\documentclass{article}

\usepackage{spconf}
\usepackage{amsfonts}   
\usepackage{amssymb}  
\usepackage{amsmath}  
\usepackage{amsthm}
\usepackage{color}
\usepackage{graphicx}
\usepackage{url}  
\usepackage{cite}
\usepackage{enumitem}

\newcommand{\G}{\mathcal{G}}
\newcommand{\hg}{\mathcal{H}}
\newcommand{\V}{\mathcal{V}}
\newcommand{\E}{\mathcal{E}}
\newcommand{\set}{\mathcal{S}}

\newcommand{\R}{\mathbb{R}}

\newcommand{\x}{\mathbf{x}}

\newcommand{\vol}{\mathrm{vol}}
\newcommand{\cut}{\mathrm{cut}}

\newtheorem{prop}{Proposition}
\newtheorem{lemma}{Lemma}

\theoremstyle{remark}
\newtheorem*{remark}{Remark}

\title{Hypergraphs with Edge-Dependent Vertex Weights: Spectral Clustering based on the $1$-Laplacian}
\name{Yu Zhu, Boning Li, and Santiago Segarra
\thanks{This work was supported by NSF under award CCF-2008555. 
E-mails: \{yz126, boning.li, segarra\}@rice.edu      }}
\address{Dept. of Electrical and Computer Engineering, Rice University, USA}

\ninept

\begin{document}
\maketitle

\begin{abstract}
We propose a flexible framework for defining the $1$-Laplacian of a hypergraph that incorporates edge-dependent vertex weights. 
These weights are able to reflect varying importance of vertices within a hyperedge, thus conferring the hypergraph model higher expressivity than homogeneous hypergraphs. 
We then utilize the eigenvector associated with the second smallest eigenvalue of the hypergraph $1$-Laplacian to cluster the vertices. 
From a theoretical standpoint based on an adequately defined normalized Cheeger cut, this procedure is expected to achieve higher clustering accuracy than that based on the traditional Laplacian.
Indeed, we confirm that this is the case using real-world datasets to demonstrate the effectiveness of the proposed spectral clustering approach. 
Moreover, we show that for a special case within our framework, the corresponding hypergraph $1$-Laplacian is equivalent to the $1$-Laplacian of a related graph, whose eigenvectors can be computed more efficiently, facilitating the adoption on larger datasets.
\end{abstract}


\section{Introduction}\label{s:intro}

Clustering is the task of dividing a set of entities into several groups such that entities in the same group are more similar to each other than to those in other groups.
In graph clustering or partitioning, the vertices of a graph correspond to the entities and the edge weights encode pairwise similarities that can also be understood as the cost of assigning the endpoints of an edge to different clusters.
In this way, graph clustering can be cast as an optimization problem that aims to minimize the total cost of cutting the edges across clusters~\cite{von2007tutorial, buhler2009spectral}.

Although graphs are widely used in a myriad of machine learning tasks, they are limited to representing \emph{pairwise} interactions.
By contrast, in many real-world applications the entities engage in \emph{higher-order} relations~\cite{battiston2020networks, schaub2021signal,roddenberry2021principled}.
Examples of these multi-way relations include authors collaborating together to generate a manuscript~\cite{chitra2019random}, 
multiple customers who bought the same product~\cite{li2018tail},
or even multiple news articles containing the same key words~\cite{hayashi2020hypergraph, zhu2021co}.
Such relations can be modeled by hypergraphs, where the notion of an edge is generalized to a hyperedge that can connect more than two vertices~\cite{battiston2020networks, schaub2021signal}.
Classical hypergraph partitioning approaches mimic graph partitioning and consider homogeneous hyperedge cuts --- an identical cost is charged if a hyperedge is cut no matter \emph{how} it is cut~\cite{hein2013total}.

To leverage the fact that different subsets of vertices in a hyperedge may have different structural importance, the concept of an inhomogeneous hyperedge cut has been proposed~\cite{li2017inhomogoenous}. 
In this case, each hyperedge $e$ is associated with a function $w_e:2^e\to\R_{\geq 0}$ that assigns costs to every possible cut of the hyperedge where $2^e$ denotes the power set of $e$.
The weight $w_e(\set)$ indicates the cost of partitioning $e$ into the two subsets $\set$ and $e\setminus\set$.
In particular, when $w_e$ is a submodular function, the corresponding model is called a submodular hypergraph which has desirable mathematical properties, making it convenient for theoretical analysis. 
A series of results in graph spectral theory~\cite{chung92spectral} including $p$-Laplacians and Cheeger inequalities have been generalized to submodular hypergraphs~\cite{li2018submodular, yoshida2019cheeger}.

Naturally, the choice of the weight function $w_e$ has a large practical effect on the clustering of a hypergraph and its subsequent downstream tasks.
Most existing works consider cardinality-based functions~\cite{zhou2006learning,agarwal2006higher,veldt2020hypergraph} in which the value of $w_e(\set)$ depends only on the cardinalities of $\set$ and $e\setminus\set$ such as $w_e(\set)=|\set|\cdot|e\setminus\set|$.
However, cardinality-based functions are limited in the sense that they cannot incorporate information regarding the different contribution of vertices to a hyperedge.
To capture such information, we leverage edge-dependent vertex weights (EDVWs): Every vertex $v$ is associated with a weight $\gamma_e(v)$ for each incident hyperedge $e$ that reflects the contribution of $v$ to $e$~\cite{chitra2019random}.
The hypergraph model with EDVWs is highly relevant in practice.
Going back to our earlier examples, EDVWs can be used to model the author positions in a co-authored manuscript~\cite{chitra2019random}, the quantity of a product bought by a user in an e-commerce system~\cite{li2018tail}, or the relevance of a word to a document in text mining~\cite{hayashi2020hypergraph, zhu2021co}.
In this context, we provide a principled way of defining weight functions $w_e$ based on the EDVWs that guarantee submodularity.
Lastly, even when given $w_e$, finding the optimal (balanced) graph cut is NP-hard~\cite{wagner1993between}.
A common workaround is to rely on spectral relaxations based on the combinatorial ($2$-)Laplacian~\cite{von2007tutorial}.
Here we outperform this standard approach by leveraging the fact that the second eigenvalue of the $1$-Laplacian is equal to the Cheeger constant~\cite{li2018submodular, chang2016spectrum, chang2017nodal, tudisco2018nodal}.

\vspace{2mm}
\noindent
\textbf{Contribution.}
The contributions of this paper can be summarized as follows:\\
1) We provide a framework to define hyperedge weight functions $w_e$ based on EDVWs that guarantee submodularity.\\
2) We consider a particular choice of $w_e$ in the proposed framework and show that its corresponding hypergraph $1$-Laplacian reduces to the $1$-Laplacian of a graph obtained via clique expansion.\\
3) We leverage the eigenvector associated with the second smallest eigenvalue of the hypergraph $1$-Laplacian to cluster the vertices and validate the effectiveness of this method using real-world datasets.

\section{Preliminaries}\label{s:pre}

We briefly review some basic mathematical concepts in Section~\ref{ss:math} and then introduce the submodular hypergraph model and its corresponding $1$-Laplacian in Section~\ref{ss:submodular_hg}. 
Throughout the paper we assume that the hypergraph is connected.

\subsection{Mathematical preliminaries}\label{ss:math}

For a finite set $\set$, a set function $F:2^\set\to\R$ is called submodular if $F(\set_1\cup\{u\}) - F(\set_1) \geq F(\set_2\cup\{u\}) - F(\set_2)$ for every $\set_1\subseteq\set_2\subset\set$ and every $u\in\set\setminus\set_2$.
Considering a set function $F:2^\set\to\R$ where $\set=[N]=\{1,2,\cdots,N\}$, its Lov\'asz extension $f:\R^N\to\R$ is defined as follows.
For any $\x\in\R^N$, sort its entries in non-increasing order $x_{i_1} \geq x_{i_2} \geq \cdots \geq x_{i_N}$ and set
\begin{equation}\label{e:lovasz}
f(\x) = \textstyle\sum_{j=1}^{N-1} F(\set_j) ( x_{i_{j}} - x_{i_{j+1}} ) + F(\set)x_{i_N}
\end{equation}
where $\set_j=\{i_1,\cdots,i_j\}$ for $1\leq j<N$.
The Lov\'asz extension $f$ is convex if and only if $F$ is submodular~\cite{lovasz1983submodular, bach2013learning}.

\subsection{Submodular hypergraphs}\label{ss:submodular_hg}

Let $\hg=(\V,\E,\mu,\{w_e\})$ denote a submodular hypergraph~\cite{li2018submodular} where $\V=[N]$ is the vertex set and $\E$ represents the set of hyperedges.
The function $\mu:\V\to\R_{+}$ assigns positive weights to every vertex.
Each hyperedge $e\in\E$ is associated with a submodular function $w_e:2^e\to\R_{\geq 0}$ that assigns non-negative costs to each possible partition of the hyperedge $e$.
Moreover, $w_e$ is required to satisfy $w_e(\emptyset)=0$ and be symmetric so that $w_e(\set)=w_e(e\setminus \set)$ for any $\set\subseteq e$.
The domain of the function $w_e$ can be extended from $2^e$ to $2^{\V}$ by setting $w_e(\set)=w_e(\set\cap e)$ for any $\set\subseteq\V$, guaranteeing that the submodularity is maintained.
Note that if $w_e$ is a constant function, submodular hypergraphs reduce to homogeneous hypergraphs, where the cost of cutting a hyperedge is independent of the specific partition.

A cut is a partition of the vertex set $\V$ into two disjoint, non-empty subsets denoted by $\set$ and its complement $\bar{\set}=\V\setminus\set$.
The weight of the cut is defined as the sum of cut costs associated with each hyperedge~\cite{li2018submodular}, i.e., $\cut(\set,\bar{\set})=\sum_{e\in\E} w_e(\set)$.
The Cheeger constant~\cite{chung92spectral, li2018submodular, tudisco2018nodal} is defined as
\begin{equation}\label{e:2_cheeger_constant}
h_2 = \min_{\emptyset\subset\set\subset\V} \frac{\cut(\set,\bar{\set})}{\min\{\vol(\set), \vol(\bar{\set})\}},
\end{equation}
where $\vol(\set)=\sum_{v\in\set}\mu(v)$ denotes the volume of $\set$.
The solution to~\eqref{e:2_cheeger_constant} provides an optimal partitioning in the sense that we obtain two balanced clusters (in terms of their volume) that are only loosely connected, as captured by a small cut weight.
Although many alternative clustering formulations exist~\cite{von2007tutorial, buhler2009spectral, carlsson2013axiomatic, carlsson2018hierarchical}, in this paper we adopt the minimization of~\eqref{e:2_cheeger_constant} as our objective.

Optimally solving~\eqref{e:2_cheeger_constant} has been shown to be NP-hard for graphs~\cite{wagner1993between, buhler2009spectral, szlam2010total}, let alone weighted hypergraphs.
Hence, different relaxations have been proposed, a popular one being spectral clustering, a relaxation based on the second eigenvector of the graph (or hypergraph) Laplacian~\cite{von2007tutorial, buhler2009spectral, hayashi2020hypergraph}.
This approach is also theoretically justified through the Cheeger inequality~\cite{chung92spectral, tudisco2018nodal, li2018submodular, chitra2019random}, where $h_2$ is \emph{upper bounded} by a function of the second smallest eigenvalue of the Laplacian.
However, it has been proved that the Cheeger constant $h_2$ is \emph{equal} to the second smallest eigenvalue $\lambda_2$ of the hypergraph $1$-Laplacian, an alternative (non-linear) operator that generalizes the classical Laplacian (cf. Theorem 4.1 in~\cite{li2018submodular}).
Moreover, the corresponding partitioning can be obtained by thresholding the eigenvector associated with $\lambda_2$ (cf. Theorem 4.3 in~\cite{li2018submodular}).
The $1$-Laplacian $\triangle_1$ of a submodular hypergraph is defined as an operator that, for all $\x\in\R^N$, induces
\begin{equation}\label{e:Q1}
\langle \x, \triangle_1(\x) \rangle = \textstyle\sum_{e\in\E} f_e(\x),
\end{equation}
where $f_e$ is the Lov\'asz extension of $w_e$.
Notice that $\triangle_1$ can be alternatively defined in terms of the subdifferential of $f_e$~\cite{li2018submodular}, but the inner product definition in~\eqref{e:Q1} is more instrumental to our development.
The eigenvector of $\triangle_1$ associated with $\lambda_2$ can be obtained by minimizing  
\begin{equation}\label{e:R1}
R_1(\x) = \frac{\langle \x, \triangle_1(\x) \rangle }{\min_{c\in\R} \|\x-c\mathbf{1}\|_{1,\mu}},
\end{equation}
where $\|\x\|_{1,\mu}=\sum_{v\in\V} \mu(v) |x_v|$.
Since finding a vector $\x$ that minimizes $R_1(\x)$ is equivalent to solving~\eqref{e:2_cheeger_constant}, this minimization is also NP-hard, stemming from the non-linear nature of $\triangle_1$.
However, developing approximate solutions through this route has shown advantages for submodular hypergraphs~\cite{li2018submodular}.
In this paper, we show that this is also the case for the more expressive model of hypergraphs with EDVWs.

\begin{remark}[Graph $1$-Laplacian]
An undirected graph $\G=(\V,\E,\mathbf{A})$ with adjacency matrix $\mathbf{A}$ can be regarded as a special case of submodular hypergraphs where $w_e(\set)$ equals the edge weight if only one endpoint of the edge is contained in $\set$ and equals zero otherwise.
Under this setting, the Lov\'asz extension $f_e$ of $w_e$ for edge $e=(u,v)$ turns out to be $f_e(\x)=A_{uv}|x_u-x_v|$ according to~\eqref{e:lovasz}.
Following~\eqref{e:Q1}, we thus obtain that the \emph{graph} 1-Laplacian, which we denote by $\triangle^{(g)}_1$ for convenience, is given by
\begin{equation}\label{e:Q1g}
\langle \x, \triangle^{(g)}_1(\x)\rangle = \textstyle\frac{1}{2}\sum_{u,v\in\V} A_{uv}|x_u-x_v|.
\end{equation}
This definition is consistent with prior definitions of the $1$-Laplacian in the graph domain~\cite{szlam2010total, hein2010inverse, chang2016spectrum, chang2017nodal}.
\end{remark}

\section{Spectral clustering of hypergraphs with EDVWs based on $1$-Laplacian}\label{s:alg}

Let $\hg=(\V,\E,\mu,\kappa,\{\gamma_e\})$ be a hypergraph with EDVWs~\cite{chitra2019random} where $\V=[N]$, $\E$, and $\mu$ respectively denote the vertex set, the hyperedge set, and positive vertex weights.
The function $\kappa:\E\to\R_{+}$ assigns positive weights to hyperedges, and those weights can reflect the strength of connection. Each hyperedge $e\in\E$ is associated with a function $\gamma_e:\V\to\R_{\geq 0}$ to assign \emph{edge-dependent vertex weights}.
For $v\in e$, $\gamma_e(v)$ is positive and measures the importance of the vertex $v$ to the hyperedge $e$; for $v\notin e$, we set $\gamma_e(v)=0$.

The introduction of EDVWs enables the hypergraph to model the cases where the vertices in one hyperedge contribute differently to this hyperedge.
For example, in a co-authorship network where authors and papers are respectively modeled as vertices and hyperedges, every author generally has a different degree of contribution to a paper, usually reflected by the order of the authors.
This information is lost in homogeneous hypergraphs and is hard to be directly described by submodular hypergraphs, but it can be conveniently represented via EDVWs.
Thus, deriving a principled framework to treat EDVWs while still leveraging the benefits of submodular hypergraphs (see Section~\ref{ss:submodular_hg}) is of paramount practical importance.

\subsection{Building submodular hypergraphs from EDVWs}\label{ss:prop1}

We propose a method to define a class of submodular weight functions $w_e$ based on EDVWs.
Putting it differently, we propose a way to convert the hypergraph model with EDVWs into a submodular hypergraph by providing a recipe that generates $\{w_e\}$ from $\kappa$ and $\{\gamma_e\}$. 
We do this by exploiting the following result.

\begin{prop}\label{prop1}
Define a weight function $w_e:2^\V\to\R_{\geq 0}$ associated with hyperedge $e$ as 
\begin{equation}\label{e:w_e_prop}
	w_e(\set) = h_e(\kappa(e)) \cdot g_e(\textstyle\sum_{v\in\set}\gamma_e(v)),
\end{equation}
for all $\set\subseteq\V$.
If $h_e:\R_{+}\to\R_{+}$ is an arbitrary function and 
$g_e:[0, \sum_{v\in e} \gamma_e(v)]\to\R_{\geq 0}$ is concave, symmetric with respect to $\sum_{v\in e} \gamma_e(v)/2$ and satisfies $g_e(0)=0$, then $w_e$ is a symmetric submodular function satisfying $w_e(\emptyset)=0$.
\end{prop}

\begin{proof}
For every $\set_1\subseteq\set_2\subset\V$ and every $u\in\V\setminus\set_2$, set $a_1=\sum_{v\in\set_1}\gamma_e(v)$, $a_2=\sum_{v\in\set_2}\gamma_e(v)$ and $b=\gamma_e(u)$.
Since $g_e$ is concave, it follows from Lemma~\ref{concave_func_property} (see below) that 
\begin{align*}
g_e(\textstyle\sum_{v\in\set_1}\gamma_e(v) + \gamma_e(u)) - g_e(\textstyle\sum_{v\in\set_1}\gamma_e(v)) \geq \\  
g_e(\textstyle\sum_{v\in\set_2}\gamma_e(v) + \gamma_e(u)) - g_e(\textstyle\sum_{v\in\set_2}\gamma_e(v)).
\end{align*}
Multiplying both sides of the above inequality by the positive value $h_e(\kappa(e))$, it immediately follows that 
\begin{equation*}
w_e(\set_1\cup\{u\}) - w_e(\set_1) \geq w_e(\set_2\cup\{u\}) - w_e(\set_2).
\end{equation*}
Hence, $w_e(\set)$ defined in~\eqref{e:w_e_prop} is submodular according to the definition given in Section~\ref{ss:math}.
Moreover, $g_e(0)=0$ and the symmetry of $g_e$ respectively guarantee that $w_e(\emptyset)=0$ and $w_e(\set)=w_e(\V\setminus\set)$ for any $\set\subseteq\V$.
\end{proof}

\begin{lemma}\label{concave_func_property}
If $a_2 \geq a_1$, $b \geq 0$ and $g$ is a concave function, the inequality $g(a_1+b)-g(a_1) \geq g(a_2+b)-g(a_2)$ holds.
\end{lemma}

\begin{proof}
For the cases where $a_1=a_2$ or $b=0$, the result is trivial.
For the other cases, recall the definition of concave functions. 
The following inequality holds for any $x$ and $y$ in the domain of $g$,
\begin{equation*}
g(tx + (1-t)y) \geq tg(x) + (1-t)g(y), \quad\forall t\in[0,1].	
\end{equation*}
Set $x=a_1$, $y=a_2+b$. For $t=\frac{a_2-a_1}{a_2-a_1+b}$ and $t=\frac{b}{a_2-a_1+b}$, the above inequality respectively becomes
\begin{align*}	
g(a_1+b) &\geq \textstyle \frac{a_2-a_1}{a_2-a_1+b} g(a_1) + \frac{b}{a_2-a_1+b} g(a_2+b), \\
g(a_2)   &\geq \textstyle\frac{b}{a_2-a_1+b} g(a_1) + \frac{a_2-a_1}{a_2-a_1+b} g(a_2+b).
\end{align*}
The proof is completed by respectively adding both sides of these two inequalities together. 
\end{proof}

\noindent
Proposition~\ref{prop1} provides a framework for constructing submodular hypergraphs from EDVWs. 
Based on (the Lov\'asz extension of) the submodular edge weights, we can further define the hypergraph $1$-Laplacian [cf.~\eqref{e:Q1}] and apply spectral clustering by minimizing~\eqref{e:R1}.
Although not required by Proposition~\ref{prop1}, since $\kappa(e)$ reflects the connection strength of hyperedge $e$, it is often reasonable in practice to select a non-decreasing function for $h_e$ such as $h_e(x)=1$ and $h_e(x)=x$.
Also notice that for trivial EDVWs --- i.e., $\gamma_e(v)=1$ for all vertices $v$ and incident hyperedges $e$ --- the functions defined in~\eqref{e:w_e_prop} reduce to cardinality-based functions~\cite{veldt2020hypergraph}.

As mentioned in Section~\ref{ss:submodular_hg}, we can threshold the second eigenvector of the $1$-Laplacian to get a $2$-partition that achieves the minimum Cheeger cut $h_2$.
More precisely, given a vector $\x\in\R^N$ and a threshold $t$, a partitioning can be defined as $\set=\{v\in\V \,|\, x_v > t \}$ and its complement. 
The optimal threshold $t$ can be determined as the one that minimizes the Cheeger cut in~\eqref{e:2_cheeger_constant}.
An approximation to the computation of the eigenvector [cf.~\eqref{e:R1}] can be obtained using the inverse power method (IPM).
The IPM involves an inner optimization problem that has faster solvers for graphs than general submodular hypergraphs~\cite{li2018submodular, hein2010inverse}.
From a scalability standpoint, this raises the question of whether there exist submodular functions in the proposed framework whose corresponding \emph{hypergraph} $1$-Laplacian is identical to some \emph{graph} $1$-Laplacian. 
In the next section, we give one example that meets this condition.

\subsection{A reducible case}\label{ss:prop2}

Given a hypergraph $\hg=(\V,\E,\mu,\kappa,\{\gamma_e\})$ with EDVWs, we define its corresponding clique expansion graph $\G=(\V,\E^{(g)},\mu,\mathbf{A})$ as an undirected graph endowed with a vertex weight function $\mu$, which is constructed as follows.
The vertex set $\V$ and vertex weights $\mu$ are identical to those of $\hg$.
Each hyperedge in $\hg$ is replaced by a clique so that $\E^{(g)}=\{(u,v) \,|\, u,v\in e, u\neq v, e\in\E\}$.
Lastly, the edge weights contained in the adjacency matrix are defined as $A_{uv}=\sum_{e\in\E} h_e(\kappa(e))\gamma_e(u)\gamma_e(v)$, where $h_e$ is defined as in Proposition~\ref{prop1}.
With this construction in place, the following result holds.

\begin{prop}\label{prop2}
For any submodular hypergraph constructed from EDVWs following~\eqref{e:w_e_prop}, if for every hyperedge $e\in\E$ we select the following concave function
\begin{equation}\label{e:clique}
g_e(x) = x\cdot(\textstyle\sum_{v\in e}\gamma_e(v) - x),	
\end{equation}
then the resulting hypergraph $1$-Laplacian is identical to the $1$-Laplacian of its corresponding clique expansion graph.
\end{prop}

\begin{proof}
For any $\x\in\R^N$, sort the subset of its entries indexed by the vertices incident to hyperedge $e$ in non-increasing order $x_{i_{1}^{(e)}} \geq x_{i_{2}^{(e)}} \geq \cdots \geq x_{i_{|e|}^{(e)}}$ and define a series of sets $\set_j^{(e)}=\{i_{1}^{(e)}, \cdots, i_{j}^{(e)}\}$ for $1\leq j <|e|$.
It follows from~\eqref{e:lovasz} that the Lov\'asz extension of the defined $w_e$ can be written as 
\begin{align*}
& f_e(\x) 
\overset{(a)}{=} \textstyle\sum_{j=1}^{|e|-1} w_e(\set_j^{(e)}) ( x_{i_{j}^{(e)}} - x_{i_{j+1}^{(e)}} ), \\
& \overset{(b)}{=} \textstyle\sum_{j=1}^{|e|-1} h_e(\kappa(e))  g_e(\sum_{v\in\set_j^{(e)}}\gamma_e(v)) ( x_{i_{j}^{(e)}} - x_{i_{j+1}^{(e)}} ), \\
& \overset{(c)}{=} h_e(\kappa(e)) \textstyle\sum_{j=1}^{|e|-1} \sum_{u\in\set_j^{(e)}, v\in e\setminus\set_j^{(e)}}\gamma_e(u)\gamma_e(v) ( x_{i_{j}^{(e)}} - x_{i_{j+1}^{(e)}} ), \\
&= \textstyle \frac{h_e(\kappa(e))}{2} \sum_{u,v\in e} \gamma_e(u)\gamma_e(v)|x_u - x_v|,
\end{align*}
where in $(a)$ we leveraged the fact that $w_e(e)=0$, and $(b)$, $(c)$ respectively followed~\eqref{e:w_e_prop} and~\eqref{e:clique}.
Then according to~\eqref{e:Q1}, the hypergraph $1$-Laplacian $\triangle_1(\x)$ should satisfy 
\begin{equation}\label{e:proof_clique_010}
\langle \x, \triangle_1(\x) \rangle = \textstyle\sum_{e\in\E}\frac{h_e(\kappa(e))}{2}\sum_{u,v\in e} \gamma_e(u)\gamma_e(v)|x_u-x_v|.
\end{equation}
It follows from~\eqref{e:Q1g} that the $1$-Laplacian $\triangle^{(g)}_1(\x)$ of the clique expansion graph $\G$ must satisfy
\begin{align}\label{e:proof_clique_020}
&\langle\x,\triangle^{(g)}_1(\x)\rangle = \textstyle\frac{1}{2}\sum_{u,v\in\V}\sum_{e\in\E} h_e(\kappa(e))\gamma_e(u)\gamma_e(v)|x_u-x_v|, \nonumber \\
& \quad\quad\quad\quad \overset{(d)}{=} \textstyle\sum_{e\in\E}\frac{h_e(\kappa(e))}{2}\sum_{u,v\in e} \gamma_e(u)\gamma_e(v)|x_u-x_v|,
\end{align}
where in $(d)$ we exchanged the order of the two summations and used the fact that $\gamma_e(v)=0$ for $v\notin e$.
Since $\x$ has been chosen arbitrarily, a direct comparison of~\eqref{e:proof_clique_010} and~\eqref{e:proof_clique_020} reveals that $\triangle_1(\x)$ and $\triangle^{(g)}_1(\x)$ coincide.
\end{proof}

\begin{figure*}[th]
\centering
\includegraphics[scale=0.51]{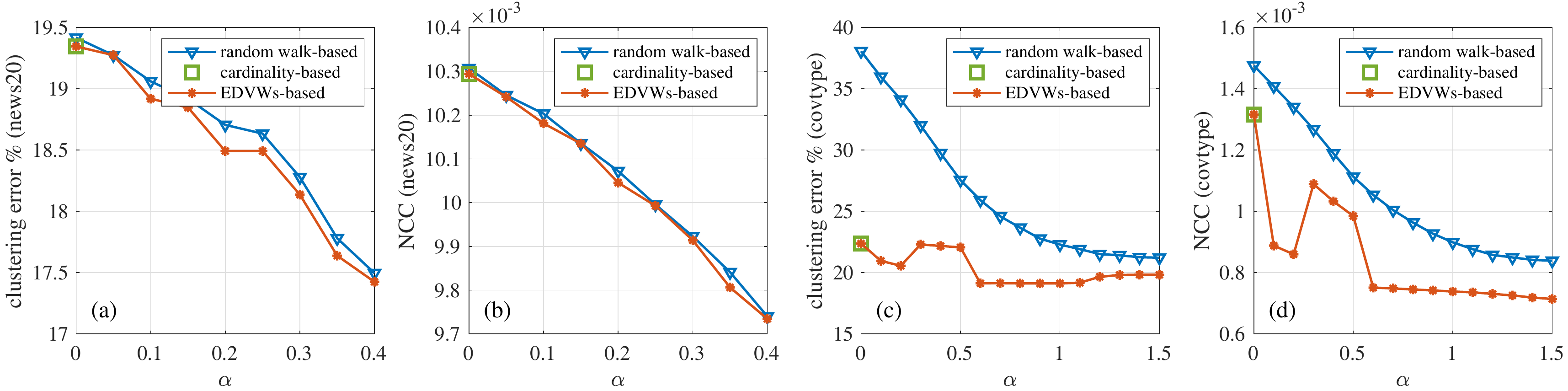}
\vspace{-7mm}
\caption{Clustering performance in two real-world datasets as a function of the parameter $\alpha$, which defines the EDVWs. (a-b) Clustering error and NCC in the $20$ Newsgroups dataset. (c-d) Clustering error and NCC in the Covertype dataset.}
\label{fig:results}
\end{figure*}

\noindent
Proposition~\ref{prop2} enables us to transform an eigenproblem in a hypergraph to an equivalent eigenproblem in a carefully constructed graph with the same number of vertices.
This allows us to leverage efficient IPM solvers~\cite{li2018submodular, hein2010inverse} when approximating the sought eigenvector of the $1$-Laplacian.
It should be noted that the choice of $g_e$ in~\eqref{e:clique} is not only computationally convenient but also meaningful in practice.
Indeed, the hyperedge cut cost achieves the maximum value when the hyperedge is partitioned into two parts that have the same sum of EDVWs.
Putting it differently, cutting a hyperedge by only shaving off an unimportant node (corresponding to a low EDVW) is lightly penalized compared to cutting the hyperedge in half, where this half is defined by the relative weights of the EDVWs.

\section{Numerical experiments}\label{s:exp}

We evaluate the performance of the proposed spectral clustering method for hypergraphs with EDVWs by focusing on the $2$-partition case.\footnote{The code needed to replicate the experiments here presented can be found at {\scriptsize{\url{https://github.com/yuzhu2019/hg_l1}}}.}
We compare the proposed method with the more standard spectral clustering approach based on the classical Laplacian and a $1$-Laplacian approach that ignores the EDVWs.

We summarize the proposed method (termed as EDVWs-based) as follows:
(i)~We build a submodular hypergraph from EDVWs according to Proposition~\ref{prop1} where we select $h_e(x)=x$ and $g_e$ as in~\eqref{e:clique}. 
For each vertex $v$, we set its weight $\mu(v)$ to its degree defined as $\deg(v)=\sum_{e\in\E:v\in e}\vartheta_e$ where $\vartheta_e=\max_{\set\subseteq e} w_e(\set)$.
For this choice of vertex weights, the objective function in~\eqref{e:2_cheeger_constant} is called normalized Cheeger cut (NCC)~\cite{buhler2009spectral}.
(ii)~We project the constructed submodular hypergraph onto its corresponding clique expansion graph as explained in Section~\ref{ss:prop2}.
(iii)~We compute the second eigenvector of the corresponding graph $1$-Laplacian using the IPM~\cite{hein2010inverse}.
(iv)~We threshold the obtained eigenvector using the optimal threshold value (see the last paragraph in Section~\ref{ss:prop1}), which achieves the minimum NCC.

\vspace{0.5mm}
\noindent {\bf Datasets.} 
We evaluate our performance on two widely used real-world datasets. \\
\emph{$20$ Newsgroups}\footnote{\scriptsize{\url{http://qwone.com/~jason/20Newsgroups/}}}: This dataset contains documents in different categories.
For our $2$-partition case, we consider the documents in categories `rec.motorcycles' and `rec.sport.hockey' and the $100$ most frequent words in the corpus after removing stop words and words appearing in $>10\%$ and $<0.2\%$ of the documents.
We then remove documents containing less than $5$ selected words, leaving us with $1406$ documents.
A document (vertex) belongs to a word (hyperedge) if the word appears in the document.
The EDVWs are taken as the corresponding tf-idf (term frequency-inverse document frequency) values~\cite{leskovec2020mining} to the power of $\alpha$, where $\alpha$ is an adjustable parameter.\\
\emph{Covertype}\footnote{\scriptsize{\url{https://archive.ics.uci.edu/ml/datasets/covertype}}}: This dataset contains areas of different forest cover types. 
We consider two cover types (types $4$ and $5$ in the dataset) and all numerical features.
Each numerical feature is first quantized into $20$ bins of equal size and then mapped to hyperedges.
The resulting hypergraph has $12240$ vertices and $196$ hyperedges. 
For each hyperedge (bin), we compute the distance between each feature value in this bin and their median, and then normalize these distances to the range $[0,1]$.
The EDVWs are computed as $\exp(-\alpha\cdot \text{distance})$.
In this way, larger EDVWs are given to nodes whose feature values are close to the typical feature value in that corresponding bin (hyperedge).

Following the setting in~\cite{hayashi2020hypergraph}, for both datasets we set the hyperedge weight $\kappa(e)$ to the standard deviation of the EDVWs $\gamma_e(v)$ for all $v\in\V$. 

\vspace{0.5mm}
\noindent{\bf Baselines.} 
We compare our proposed approach with two baseline methods:\\
\emph{Random walk-based}:
The paper~\cite{hayashi2020hypergraph} defines a hypergraph Laplacian based on random walks with EDVWs.
We compute the second eigenvector of the normalized hypergraph Laplacian proposed in~\cite{hayashi2020hypergraph} and then threshold it to get the partitioning.\\
\emph{Cardinality-based}:
In the description of our datasets, when $\alpha=0$ we are effectively ignoring the EDVWs and the corresponding submodular weight functions $w_e$ become cardinality-based.

\vspace{0.5mm}
\noindent{\bf Results.} 
The results are shown in Fig.~\ref{fig:results}, where the clustering error is computed as the fraction of incorrectly clustered samples.
Focusing first on the $20$ Newsgroups dataset (a-b), we can see that the proposed method achieves better clustering performance than the random-walk method based on the classical Laplacian. 
This gain can be explained by our method yielding a smaller NCC value.
More conspicuously, when EDVWs are ignored ($\alpha=0$), the clustering performance is severely deteriorated, underscoring the importance of the added modeling flexibility of hypergraphs with EDVWs.
Shifting attention to the Covertype dataset (c-d), the same trends are observed. 
However, in this case there are marked performance differences between our method and the one based on random walks, highlighting the value of considering the non-linear $1$-Laplacian in spectral clustering.

\section{Conclusions}\label{s:end}

We developed a framework for defining submodular weight functions based on EDVWs and provided an example that can be reduced to a clique expansion graph.
Numerical results indicated that the combination of the hypergraph $1$-Laplacian and EDVWs helps improve clustering performance.
Future research avenues include:
(i)~Studying if there exist other choices in the context of Proposition~\ref{prop1} such that the corresponding hypergraph $1$-Laplacian coincides with some graph $1$-Laplacian,
(ii)~Developing efficient computation methods for the hypergraph $1$-Laplacian's eigenvectors that are applicable to general submodular weight functions, and
(iii)~Designing multi-way partitioning algorithms based on non-linear Laplacians~\cite{buhler2009spectral}.


\vfill\pagebreak
\bibliographystyle{IEEEbib}
\bibliography{references}

\begin{thebibliography}{10}

\bibitem{von2007tutorial}
Ulrike Von~Luxburg,
\newblock ``A tutorial on spectral clustering,''
\newblock {\em Statistics and computing}, vol. 17, no. 4, pp. 395--416, 2007.

\bibitem{buhler2009spectral}
Thomas B{\"u}hler and Matthias Hein,
\newblock ``Spectral clustering based on the graph p-{Laplacian},''
\newblock in {\em ICML}, 2009, pp. 81--88.

\bibitem{battiston2020networks}
Federico Battiston, Giulia Cencetti, Iacopo Iacopini, Vito Latora, Maxime
  Lucas, Alice Patania, Jean-Gabriel Young, and Giovanni Petri,
\newblock ``Networks beyond pairwise interactions: Structure and dynamics,''
\newblock {\em Physics Reports}, vol. 874, pp. 1--92, 2020.

\bibitem{schaub2021signal}
Michael~T Schaub, Yu~Zhu, Jean-Baptiste Seby, T~Mitchell Roddenberry, and
  Santiago Segarra,
\newblock ``Signal processing on higher-order networks: Livin' on the edge...
  and beyond,''
\newblock {\em Signal Processing}, vol. 187, pp. 108149, 2021.

\bibitem{roddenberry2021principled}
T.~Mitchell Roddenberry, Nicholas Glaze, and Santiago Segarra,
\newblock ``Principled simplicial neural networks for trajectory prediction,''
\newblock in {\em ICML}, 2021, pp. 9020--9029.

\bibitem{chitra2019random}
Uthsav Chitra and Benjamin Raphael,
\newblock ``Random walks on hypergraphs with edge-dependent vertex weights,''
\newblock in {\em ICML}, 2019, pp. 1172--1181.

\bibitem{li2018tail}
Jianbo Li, Jingrui He, and Yada Zhu,
\newblock ``E-tail product return prediction via hypergraph-based local graph
  cut,''
\newblock in {\em KDD}, 2018, pp. 519--527.

\bibitem{hayashi2020hypergraph}
Koby Hayashi, Sinan~G Aksoy, Cheong~Hee Park, and Haesun Park,
\newblock ``Hypergraph random walks, {Laplacians}, and clustering,''
\newblock in {\em CIKM}, 2020, pp. 495--504.

\bibitem{zhu2021co}
Yu~Zhu, Boning Li, and Santiago Segarra,
\newblock ``Co-clustering vertices and hyperedges via spectral hypergraph
  partitioning,''
\newblock {\em arXiv preprint arXiv:2102.10169}, 2021.

\bibitem{hein2013total}
Matthias Hein, Simon Setzer, Leonardo Jost, and Syama~Sundar Rangapuram,
\newblock ``The total variation on hypergraphs - {L}earning on hypergraphs
  revisited,''
\newblock {\em NIPS}, vol. 26, pp. 2427--2435, 2013.

\bibitem{li2017inhomogoenous}
Pan Li and Olgica Milenkovic,
\newblock ``Inhomogoenous hypergraph clustering with applications,''
\newblock in {\em NIPS}, 2017, pp. 2305--2315.

\bibitem{chung92spectral}
Fan~RK Chung,
\newblock {\em Spectral Graph Theory}, vol.~92,
\newblock American Mathematical Soc.

\bibitem{li2018submodular}
Pan Li and Olgica Milenkovic,
\newblock ``Submodular hypergraphs: {p-Laplacians}, {Cheeger} inequalities and
  spectral clustering,''
\newblock in {\em ICML}, 2018, pp. 3014--3023.

\bibitem{yoshida2019cheeger}
Yuichi Yoshida,
\newblock ``{Cheeger} inequalities for submodular transformations,''
\newblock in {\em Proceedings of the Thirtieth Annual ACM-SIAM Symposium on
  Discrete Algorithms}, 2019, pp. 2582--2601.

\bibitem{zhou2006learning}
Dengyong Zhou, Jiayuan Huang, and Bernhard Sch{\"o}lkopf,
\newblock ``Learning with hypergraphs: {Clustering}, classification, and
  embedding,''
\newblock {\em NIPS}, vol. 19, pp. 1601--1608, 2006.

\bibitem{agarwal2006higher}
Sameer Agarwal, Kristin Branson, and Serge Belongie,
\newblock ``Higher order learning with graphs,''
\newblock in {\em ICML}, 2006, pp. 17--24.

\bibitem{veldt2020hypergraph}
Nate Veldt, Austin~R Benson, and Jon Kleinberg,
\newblock ``Hypergraph cuts with general splitting functions,''
\newblock {\em arXiv preprint arXiv:2001.02817}, 2020.

\bibitem{wagner1993between}
Dorothea Wagner and Frank Wagner,
\newblock ``Between min cut and graph bisection,''
\newblock in {\em International Symposium on Mathematical Foundations of
  Computer Science}. Springer, 1993, pp. 744--750.

\bibitem{chang2016spectrum}
Kung-Ching Chang,
\newblock ``Spectrum of the 1-{Laplacian} and {Cheeger's} constant on graphs,''
\newblock {\em Journal of Graph Theory}, vol. 81, no. 2, pp. 167--207, 2016.

\bibitem{chang2017nodal}
Kung-Ching Chang, Sihong Shao, and Dong Zhang,
\newblock ``Nodal domains of eigenvectors for 1-{Laplacian} on graphs,''
\newblock {\em Advances in Mathematics}, vol. 308, pp. 529--574, 2017.

\bibitem{tudisco2018nodal}
Francesco Tudisco and Matthias Hein,
\newblock ``A nodal domain theorem and a higher-order {Cheeger} inequality for
  the graph p-{Laplacian},''
\newblock {\em Journal of Spectral Theory}, vol. 8, no. 3, pp. 883--908, 2018.

\bibitem{lovasz1983submodular}
L{\'a}szl{\'o} Lov{\'a}sz,
\newblock ``Submodular functions and convexity,''
\newblock in {\em Mathematical Programming The State of the Art}, pp. 235--257.
  Springer, 1983.

\bibitem{bach2013learning}
Francis Bach,
\newblock ``Learning with submodular functions: A convex optimization
  perspective,''
\newblock {\em Foundations and Trends{\textregistered} in Machine Learning},
  vol. 6, no. 2--3, pp. 145--373, 2013.

\bibitem{carlsson2013axiomatic}
Gunnar Carlsson, Facundo Mémoli, Alejandro Ribeiro, and Santiago Segarra,
\newblock ``Axiomatic construction of hierarchical clustering in asymmetric
  networks,''
\newblock in {\em IEEE Intl. Conf. on Acoustics, Speech and Signal Process.
  (ICASSP)}, 2013, pp. 5219--5223.

\bibitem{carlsson2018hierarchical}
Gunnar Carlsson, Facundo Mémoli, Alejandro Ribeiro, and Santiago Segarra,
\newblock ``Hierarchical clustering of asymmetric networks,''
\newblock {\em Adv. Data Anal. Classif.}, vol. 12, pp. 65--105, 2018.

\bibitem{szlam2010total}
Arthur Szlam and Xavier Bresson,
\newblock ``Total variation and {Cheeger} cuts,''
\newblock in {\em ICML}, 2010, pp. 1039--1046.

\bibitem{hein2010inverse}
Matthias Hein and Thomas B{\"u}hler,
\newblock ``An inverse power method for nonlinear eigenproblems with
  applications in 1-spectral clustering and sparse {PCA},''
\newblock {\em NIPS}, vol. 23, pp. 847--855, 2010.

\bibitem{leskovec2020mining}
Jure Leskovec, Anand Rajaraman, and Jeffrey~David Ullman,
\newblock {\em Mining of massive datasets},
\newblock Cambridge university press, 2020.

\end{thebibliography}

\end{document}